\documentclass[english]{article}
\usepackage[utf8]{inputenc} 
\usepackage[T1]{fontenc}    
\usepackage{hyperref}       
\usepackage{url}            
\usepackage{booktabs}       
\usepackage{amsfonts}       
\usepackage{nicefrac}       
\usepackage{microtype}      
\usepackage{xcolor}         
\usepackage{fancybox}
\usepackage{calc}
\usepackage{url}
\usepackage{multirow}
\usepackage{amsmath}
\usepackage{amsthm}
\usepackage{amssymb}
\usepackage{graphicx}
\PassOptionsToPackage{normalem}{ulem}
\usepackage{ulem}

\makeatletter

\providecommand{\tabularnewline}{\\}

\theoremstyle{plain}
\newtheorem{thm}{\protect\theoremname}
\theoremstyle{plain}
\newtheorem{prop}[thm]{\protect\propositionname}

\usepackage{authblk}
\usepackage[preprint]{neurips_2024}
\usepackage{algorithm}
\usepackage{algorithmic}
\usepackage{hyperref}

\makeatother

\usepackage{babel}
\providecommand{\propositionname}{Proposition}
\providecommand{\theoremname}{Theorem}
\renewcommand\cite{\citep}

\begin{document}
\title{Multi-Reference Preference Optimization for Large Language Models}

\author{
\textbf{Hung Le}$^{1}$ \quad \textbf{Quan Tran}$^{2}$ \quad \textbf{Dung Nguyen}$^1$ \quad \textbf{Kien Do}$^1$ \quad \textbf{Saloni Mittal}$^2$ \\ \quad \textbf{Kelechi Ogueji}$^2$   \quad \textbf{Svetha Venkatesh}}

\affil[1]{Applied AI Institute, Deakin University, Geelong, Australia} 
\affil[1]{\texttt {\{thai.le,dung.nguyen,k.do,svetha.venkatesh\}@deakin.edu.au}}
\affil[2]{ServiceNow Research, USA}
\affil[2]{\texttt{\{hungquan.tran,saloni.mittal,kelechi.ogueji\}@servicenow.com}}
\renewcommand\Authands{ \textbf{and} } 
\maketitle
\begin{abstract}
How can Large Language Models (LLMs) be aligned with human intentions
and values? A typical solution is to gather human preference on model outputs and finetune the LLMs accordingly while
ensuring that updates do not deviate too far from a
reference model. Recent approaches, such as direct preference optimization
(DPO), have eliminated the need for unstable and sluggish reinforcement
learning optimization by introducing close-formed supervised losses.
However, a significant limitation of the current approach is its design
for \emph{a single reference model only, neglecting to leverage the
collective power of numerous pretrained LLMs}. To overcome this limitation,
we introduce a novel closed-form formulation for \emph{direct preference
optimization using multiple reference models}. The resulting
algorithm, Multi-Reference Preference Optimization (MRPO), leverages
broader prior knowledge from diverse reference models, substantially
enhancing preference learning capabilities compared to the single-reference
DPO. Our experiments demonstrate that LLMs finetuned with MRPO generalize
better in various preference data, regardless of data scarcity or
abundance. Furthermore, MRPO effectively finetunes LLMs to exhibit
superior performance in several downstream natural language processing
tasks such as GSM8K and TruthfulQA.
\end{abstract}

\section{Introduction}

Large Language Models (LLMs) have emerged as powerful tools in natural
language processing, capable of generating human-like text and performing
a myriad of language-related tasks \cite{lewkowycz2022solving,achiam2023gpt,touvron2023llama}.
However, aligning these models with human intentions and values remains
a challenging endeavor \cite{wang2023aligning}. Aligning LLMs with
curated human feedback emerges as a critical solution to guide LLM
response behavior and address this challenge.  Preference models like the Bradley-Terry model \cite{bradley1952rank} are often used to measure the alignment of reward functions with empirical preference data, facilitating an alignment framework using reinforcement learning with human feedback (RLHF \cite{christiano2017deep}). The framework
aims to optimize the preference models (maximizing preference reward)
while ensuring that LLM updates do not stray too far from
a base reference LLM model (minimizing a Kullback-Leibler (KL) divergence).
While RLHF has been successful in enhancing the helpfulness and accuracy of model-generated content \cite{ouyang2022training,stiennon2020learning},
it is unstable, complicated, and resource-intensive. 

Recent advancements, such as direct preference optimization (DPO \cite{rafailov2023direct})
and other likelihood-based preference learning \cite{zhao2023slic,azar2023general,ethayarajh2024kto,chen2024self},
have sought to replace the cumbersome RLHF with closed-form supervised
losses. Although they have demonstrated impressive performance compared
to RLHF and supervised finetuning (SFT), their exclusive reliance
on a single reference model restricts their potential, overlooking
the advantages of harnessing multiple pretrained LLMs. Incorporating
multiple LLMs to constrain the update of the optimized LLM towards
a preference dataset results in a model that reflects the characteristics
of all reference models while satisfying human preference. This is
increasingly important as the open-source community consistently introduces
new pretrained/SFT LLMs of varying scales, trained on diverse datasets
\cite{touvron2023llama,penedo2023refinedweb,jiang2023mistral}. It
underscores the necessity for a solution that employs multiple references for LLM finetuning, enabling the distillation
of knowledge from existing LLMs to enhance the alignment training stage.
\emph{Unfortunately, none of the prior works have
proposed a solution for utilizing multiple reference LLMs in direct preference optimization. }

The absence of such solutions stems from three challenges in formulating
closed-form multiple-reference preference learning. Firstly, deriving
a closed-form solution for the RLHF objective with multiple referencing
constraints is nontrivial due to the non-linearity of multiple KL
terms. Secondly, reference models with varying architecture, size,
and pretraining data may produce diverging outputs
given the same input. This divergence could potentially confuse the
learning process, leading to unstable training. Thirdly, determining
the contribution of each reference model during training poses a challenge,
requiring extensive tuning. In this paper, we tackle these three challenges,
presenting a simple and viable framework for direct preference optimization
utilizing multiple reference models.

To address the non-linearity of KL divergence, we propose maximizing
a simpler surrogate lower bound that allows for the derivation of
a novel closed-form solution incorporating multiple reference models.
Our solution is theoretically and empirically proven
superior to combining multiple DPO losses. Next, we propose
a clipped trust-regions optimization (CTRO) to address the second
challenge. By clipping the log probability of diverging reference
policy, we force the mismatch to be minimal to facilitate stable training
while retaining useful information from the reference policy
to guide the optimization. More importantly, the clipping rate is
dynamically adjusted according to the predicted likelihood of the
data, enabling a more adaptable update. Lastly, to automate the process
of determining the contribution of each reference model, we introduce
a dynamic mechanism (ARWC) to calculate the weight of each KL term
based on the confidence of the referencing LLMs. 

Our holistic framework, dubbed Multiple Reference Preference Optimization
(MRPO), undergoes throughout evaluation across various tasks. In preference
learning tasks involving 6 preference datasets, MRPO demonstrates
significant superiority over DPO and a naive combination of multiple
DPO losses, especially when preference data is limited with improvement
of up to 7\%. Furthermore, on general language understanding benchmarks
like the HuggingFace Open LLM Leaderboard \cite{beeching2023open},
MRPO exhibits average enhancements of 3-4\% compared to SFT and
1.2\% compared to DPO. Certain tasks show more than 5\% improvements over DPO.  Importantly, these enhancements are evident across various configurations, including different numbers of reference models (2 or 3) and sizes of LLMs (1B or 7B). Finally, we perform
a comprehensive ablation study to demonstrate the efficacy of CTRO
and ARWC mechanisms.

\section{Background}

\subsection{Problem Formulation and Notations}

We rely on \cite{azar2023general} to formally define the problem
and notations. Given an input $x\in\mathcal{X}$ where $\mathcal{X}$
is the finite space of input texts, a policy $\pi$ models a conditional
probability distribution $\pi(y|x)$ where $y\in\mathcal{Y}$ is the
output in the finite space of output texts. From a given $\pi$ and
$x$, we can sample an output as $y\sim\pi\left(\cdot|x\right)$.
Preference data is generated by sampling two outputs $(y,y'|x)$ from
policies $\pi$ and $\mu$ and presenting them to an agent, normally a human,
for rating to indicate which one is preferred. For example, $y\succ y'$
denotes $y$ is preferred to $y'$. A preference dataset is then denoted
as $\mathcal{D}=\left\{ y_{w}^{i},y_{l}^{i}|x^{i}\right\} _{i=1}^{N}$
where $N$ is the number of data points, $y_{w}$ and $y_{l}$ denote
the preferred (chosen) and dispreferred (rejected), respectively. Assuming that there exists a true
model of preference of the agent $p^{*}\left(y\succ y'|x\right)$
that assigns the agent's probability of $y$ being preferred to $y'$
given $x$. Using dataset $\mathcal{D}$, our goal is to find a policy
$\pi$ maximizing the expected preference while being close to a reference
policy $\pi_{ref}$, which results in the following optimization problem:
\begin{align}
 & \underset{\pi}{\max}\underset{y\sim\pi\left(\cdot|x\right)}{\underset{y'\sim\mu\left(\cdot|x\right)}{\underset{x\sim\rho}{\mathbb{E}}}}\left[\varPsi\left(p^{*}\left(y\succ y'|x\right)\right)\right]-\text{\ensuremath{\beta}}D_{KL}\left(\pi\left\Vert \pi_{ref}\right.\right)\label{eq:pref1}
\end{align}
where $\rho$ is the input distribution, $\varPsi$ is a scaled function, $D_{KL}$ is the Kullback--Leibler
divergence and $\beta$ is a hyperparameter. Usually, $\pi$ is initialized as $\pi_{ref}$ for stable optimization. 

\subsection{Preference Learning with Reward Function and Reinforcement Learning}

In this approach, Bradley-Terry model \cite{bradley1952rank} is employed
as the preference model:
\begin{equation}
p\left(y\succ y'|x\right)=\sigma\left(r_{\theta}\left(x,y\right)-r_{\theta}\left(x,y'\right)\right)
\end{equation}
where $\sigma$ denotes the sigmoid function and $r:\mathcal{X}\times\mathcal{Y}\rightarrow\mathbb{R}$
is a reward model parameterized by $\theta,$ which assigns a scalar
score to indicate the suitability of output $y$ for input $x$. In
earlier works \cite{christiano2017deep}, the reward model is trained
on $\mathcal{D}$ to minimize the negative log-likelihood loss:
\begin{equation}
\mathcal{L}_{R}=-\mathbb{E}_{x,y_{w},y_{l}\sim\mathcal{D}}\left[\log\,\sigma\left(r_{\theta}\left(x,y_{w}\right)-r_{\theta}\left(x,y_{l}\right)\right)\right]
\end{equation}
Given a trained reward model $r$, and the scaled function as $\varPsi\left(q\right)=\log\left(\frac{q}{1-q}\right) \forall q: 0<q<1$,
the objective in Eq. \ref{eq:pref1} can be rewritten as:
\begin{align}
 & \underset{\pi}{\max}\underset{y\sim\pi\left(\cdot|x\right)}{\underset{x\sim\rho}{\mathbb{E}}}\left[r\left(x,y\right)\right]-\text{\ensuremath{\beta}}D_{KL}\left(\pi\left\Vert \pi_{ref}\right.\right)\label{eq:pref1-rlhf}
\end{align}
This RLHF objective is employed to train LLMs such as Instruct-GPT \cite{ouyang2022training} using PPO \cite{schulman2017proximal}. 

\subsection{Direct Preference Optimization}

Reward training and RL finetuning require significant resources and
can be cumbersome. Recent approaches circumvent these challenges by
directly optimizing the policy via minimizing a preference-based
negative log-likelihood loss \cite{rafailov2023direct}:

\begin{equation}
\mathcal{L}_{DPO}=-\mathbb{E}_{x,y_{w},y_{l}\sim\mathcal{D}}\left[\log\sigma\left(\beta\log\frac{\pi_{\theta}\left(y_{w}|x\right)}{\pi_{ref}\left(y_{w}|x\right)}-\beta\log\frac{\pi_{\theta}\left(y_{l}|x\right)}{\pi_{ref}\left(y_{l}|x\right)}\right)\right]\label{eq:pref1-dpo}
\end{equation}
The term $r_{\theta}\left(x,y|\pi_{ref}\right)=\beta\log\frac{\pi_{\theta}\left(y|x\right)}{\pi_{ref}\left(y|x\right)}$
plays the role of an implicit reward. The authors in \cite{rafailov2023direct}
proved that minimizing this loss is equivalent to solving the optimization
problem in Eq. \ref{eq:pref1-rlhf}.

\section{Method}\label{sec:method}

\subsection{Multi-Reference Preference Optimization}

In this paper, we are focused on situations involving $K$ reference
policies $\left\{ \pi_{ref}^{k}\right\} _{k=1}^{K}$. Therefore, extending
from Eq. \ref{eq:pref1-rlhf}, our objective can be formulated as
a multi-reference RLHF objective:
\begin{align}
\underset{\pi}{\max}\underset{y\sim\pi\left(\cdot|x\right)}{\underset{x\sim\rho}{\mathbb{E}}}\left[r\left(x,y\right)\right]-\text{\ensuremath{\beta}}\left(\sum_{k=1}^{K}\alpha_{k}D_{KL}\left(\pi\left\Vert \pi_{ref}^{k}\right.\right)\right)\label{eq:prefm-rlhf-1}
\end{align}
where $\alpha_{k}$ are weighting coefficients for each reference
policy and $1=\sum_{k=1}^{K}\alpha_{k}$. Without loss of generality,
we denote $\pi_{ref}^{1}$ as the initializing reference policy for
the main policy $\pi_{\theta}$. This objective
was explored in previous studies, leading to enhancements in pure
RL problems \cite{le2022learning}.

However, addressing this optimization problem in LLMs through reward
learning and RL finetuning poses similar challenges to Eq. \ref{eq:pref1-rlhf}.
Hence, we propose an alternative approach that leverages direct preference
optimization for the scenario involving multiple reference policies.
We aim to find a closed-form solution for the multi-reference RLHF
objective in Eq. \ref{eq:prefm-rlhf-1}. Unfortunately, deriving an exact closed-form solution is challenging due to the
nonlinearity of $D_{KL}$ terms. To circumvent this, we suggest obtaining a \emph{closed-form
solution for a surrogate objective}, which serves as a lower bound
for the multi-reference RLHF objective. We summarize our findings as
a proposition below. 
\begin{prop}\label{prop1}
The following policy is the optimum for a lower bound of the
RLHF objective (Eq. \ref{eq:prefm-rlhf-1}):

\[
\pi^{*}\left(y|x\right)=\frac{1}{Z\left(x\right)}\tilde{\pi}{}_{ref}\left(y|x\right)\exp\left(\frac{1}{\beta}r\left(x,y\right)\right)
\]
where $\tilde{\pi}_{ref}(y|x)=\left(\sum_{k=1}^{K}\frac{\alpha_{k}}{\pi_{ref}^{k}\left(y|x\right)}\right)^{-1}$ 
and $Z\left(x\right)=\sum_{y}\tilde{\pi}_{ref}(y|x)\exp\left(\frac{1}{\beta}r\left(x,y\right)\right)$.
\end{prop}

\begin{proof}
See Appendix \ref{subsec:Lower-Bound-Close-formed}.
\end{proof}
Following the derivation in \cite{rafailov2023direct} with our proposed
optimal policy $\pi^*$, we have the associated direct preference loss function
as follow,

\begin{eqnarray}
\mathcal{L}_{MRPO} & = & -\mathbb{E}_{x,y_{w},y_{l}\sim\mathcal{D}}\left[\log\sigma\left(\beta\log\frac{\pi_{\theta}\left(y_{w}|x\right)}{\tilde{\pi}_{ref}\left(y_{w}|x\right)}-\beta\log\frac{\pi_{\theta}\left(y_{l}|x\right)}{\tilde{\pi}_{ref}\left(y_{l}|x\right)}\right)\right]\label{eq:mrpo}
\end{eqnarray}
The loss function is similar to the DPO loss (Eq. \ref{eq:pref1-dpo}), but instead of using a single reference policy $\pi_{ref}$, we substitute it with a "virtual" reference policy $\tilde{\pi}_{ref}$ that aggregates information from all multiple reference policies.

\subsection{Clipped Trust-Regions Optimization (CTRO)\label{subsec:Optimization-with-Clipped}}

An issue that may arise when multiple reference policies are involved
in the optimization process is the mismatch between the reference
policy and the main policy. This is less common with single-reference
DPO, as the main policy is initialized as the reference policy, ensuring
a small mismatch across training. However, with multiple-reference
policies, those not chosen to initialize the main policy can result
in significantly different probabilities compared to the main one,
potentially leading to unstable training and, at times, loss divergence. 

To address this issue, we propose to constrain the virtual reference
policy $\tilde{\pi}_{ref}$ in the vicinity of the initializing reference
policy $\pi_{ref}^{1}$. As such, we propose to clip the log-probability
of the other reference policies $\pi_{ref}^{k>1}$ as follows, 
\begin{equation}
\log\hat{\pi}_{ref}^{k>1}\left(y|x\right)=\min\left(\mathrm{\max}\left(\log\pi_{ref}^{k>1}\left(y|x\right),\left(1+\epsilon\right)\log\pi_{ref}^{1}\left(y|x\right)\right),\left(1-\epsilon\right)\log\pi_{ref}^{1}\left(y|x\right)\right)
\end{equation}
where $\epsilon$ defines the vicinity range around $\pi_{ref}^{1}$.
Then we will replace $\pi_{ref}^{k>1}$ with $\hat{\pi}_{ref}^{k>1}$ in the $\tilde{\pi}_{ref}$ defined in Proposition \ref{prop1}. 

Using a fixed $\epsilon$ is overly restrictive and can be suboptimal
since different data and policies may require different trust-region
ranges. Thus, we suggest an adaptive approach to define $\epsilon$
based on the predicted likelihood of the data. Essentially, if the
log probability of a reference model for a given data point is large,
indicating high reliability, a conservative update should be applied
(smaller $\epsilon$) to exploit the reference policy. Conversely,
for lower log probabilities, we may prefer more exploration, allowing
reference values to diverge from the initial $\pi_{ref}^{1}$ (bigger
$\epsilon$). That is,
\begin{equation}
\epsilon\left(y|x\right)=\epsilon_{max}\frac{\left|\sum_{k=1}^{K}\log\,\pi_{ref}^{k}\left(y|x\right)\right|}{\sum_{y'}\left|\sum_{k=1}^{K}\log\,\pi_{ref}^{k}\left(y'|x\right)\right|}
\end{equation}
where $\epsilon_{max}$ is a hyperparameter specifying the maximum
ratio for an updating range. It is worth noting that since $0\leq\pi\left(\cdot\right)\leq1$,
the log-probability is always negative, meaning a larger absolute
value of the log-probability indicates less confidence. The denominator
represents the sum of all possible output $y'$, serving as a normalization
factor. In practice, in preference learning, we only have two outputs
($y_{w}$,$y_{l}$) per input $x$ so the denominator is the sum of
two terms. 

\subsection{Adaptive Reference Weighting Coefficients (ARWC)\label{subsec:The-Choice-of}}

If we have no preference or prior knowledge of the reference policies,
we can simply use $\alpha_{k}=1/K\,\forall k$. However, if we assume
that the reference policy obtains a reasonable ability to differentiate
between $y_{w}$ and $y_{l}$ such that even when it make wrong preference
(i.e $\log y_{l}$\textgreater$\log y_{w}$) the likelihood difference
should not be too large, we can introduce an automatic mechanism to
determine the value of $\alpha_{k}$ based on the confidence of the
reference policy. Specifically, we examine the absolute difference
between the log-probability of the two outputs ($y_{w}$, $y_{l}$)
as an indicator of the policy's confidence in its ability to discriminate between two outputs. In essence, a larger difference
suggests that the policy distinguishes one output from another more
decisively. Formally, we propose to adaptively compute reference weighting
coefficients as:

\begin{equation}
\alpha_{k}=\frac{\left|\log\,\pi_{ref}^{k}\left(y_{w}|x\right)-\log\,\pi_{ref}^{k}\left(y_{l}|x\right)\right|}{\sum_{i=1}^{K}\left|\log\,\pi_{ref}^{i}\left(y_{w}|x\right)-\log\,\pi_{ref}^{i}\left(y_{l}|x\right)\right|}
\end{equation}
The coefficient is normalized across reference policies, giving greater weight to those with higher discriminative confidence.

\subsection{Comparison with Multiple DPO }

When having multiple reference policies, a naive solution for
direct preference learning is to combine multiple DPO losses (Multi-DPO):
\begin{equation}
\mathcal{L}_{Multi-DPO}=-\mathbb{E}_{x,y_{w},y_{l}\sim\mathcal{D}}\left[\sum_{k=1}^{K}\alpha_{k}\log\sigma\left(\beta\log\frac{\pi_{\theta}\left(y_{w}|x\right)}{\pi_{ref}^{k}\left(y_{w}|x\right)}-\beta\log\frac{\pi_{\theta}\left(y_{l}|x\right)}{\pi_{ref}^{k}\left(y_{l}|x\right)}\right)\right]\label{eq:mdpo}
\end{equation}
We can show that our proposed MRPO is more desirable than this MDPO
loss. To see that, we analyze the gradient wrt. the policy parameters
$\theta$ of the two losses: (1)
$\nabla_{\theta}\mathcal{L}_{MRPO} =-\beta\mathbb{E}_{x,y_{w},y_{l}\sim\mathcal{D}}\left[\sigma\right.\left(r_{\theta}\left(x,y_{l}|\tilde{\pi}_{ref}\right)-r_{\theta}\left(x,y_{w}|\tilde{\pi}_{ref}\right)\right)\nonumber \times\left.\left(\nabla_{\theta}\log\pi_{\theta}\left(y_{w}|x\right)-\nabla_{\theta}\log\pi_{\theta}\left(y_{l}|x\right)\right)\right]$, and (2)
$\nabla_{\theta}\mathcal{L}_{Multi-DPO} =-\beta\mathbb{E}_{x,y_{w},y_{l}\sim\mathcal{D}}\left[\alpha_{k}\right.\sum_{k=1}^{K}\sigma\left(r_{\theta}\left(x,y_{l}|\pi_{ref}^{k}\right)-r_{\theta}\left(x,y_{w}|\pi_{ref}^{k}\right)\right)\nonumber\times\left.\left(\nabla_{\theta}\log\pi_{\theta}\left(y_{w}|x\right)-\nabla_{\theta}\log\pi_{\theta}\left(y_{l}|x\right)\right)\right]$.

The gradients mean that the rate of increasing/decreasing the likelihood
of preferred/dispreferred outputs is scaled by the reward error $\sigma\left(r_{\theta}\left(x,y_{l}|\tilde{\pi}_{ref}\right)-r_{\theta}\left(x,y_{w}|\tilde{\pi}_{ref}\right)\right)$
or $\sum_{k=1}^{K}\alpha_{k}\sigma\left(r_{\theta}\left(x,y_{l}|\pi_{ref}^{k}\right)-r_{\theta}\left(x,y_{w}|\pi_{ref}^{k}\right)\right)$, corresponding to MRPO and Multi-DPO, respectively.
Under mild assumptions, we can have the following result:
\begin{prop}
Assume that reference policies are constrained to be relatively close
to each other, ensuring that $\left\{ d_{k}=r_{\theta}\left(x,y_{l}|\pi_{ref}^{k}\right)-r_{\theta}\left(x,y_{w}|\pi_{ref}^{k}\right)\right\} _{k=1}^{K}$
share the same sign $\forall k$, then

\[
\begin{cases}
\sigma\left(\beta\log\frac{\pi_{\theta}\left(y_{l}|x\right)}{\tilde{\pi}_{ref}\left(y_{l}|x\right)}-\beta\log\frac{\pi_{\theta}\left(y_{w}|x\right)}{\tilde{\pi}_{ref}\left(y_{w}|x\right)}\right) & \geq\sum_{k=1}^{K}\alpha_{k}\sigma\left(\beta\log\frac{\pi_{\theta}\left(y_{l}|x\right)}{\pi_{ref}^{k}\left(y_{l}|x\right)}-\beta\log\frac{\pi_{\theta}\left(y_{w}|x\right)}{\pi_{ref}^{k}\left(y_{w}|x\right)}\right)\,\,\,\forall d_{k}\geq0\\
\sigma\left(\beta\log\frac{\pi_{\theta}\left(y_{l}|x\right)}{\tilde{\pi}_{ref}\left(y_{l}|x\right)}-\beta\log\frac{\pi_{\theta}\left(y_{w}|x\right)}{\tilde{\pi}_{ref}\left(y_{w}|x\right)}\right) & \leq\sum_{k=1}^{K}\alpha_{k}\sigma\left(\beta\log\frac{\pi_{\theta}\left(y_{l}|x\right)}{\pi_{ref}^{k}\left(y_{l}|x\right)}-\beta\log\frac{\pi_{\theta}\left(y_{w}|x\right)}{\pi_{ref}^{k}\left(y_{w}|x\right)}\right)\,\,\,\forall d_{k}\leq0
\end{cases}
\]
\end{prop}

\begin{proof}
See Appendix \ref{subsec:MRPO-and-Multi-DPO}.
\end{proof}
As a result, when the reward estimations are wrong, i.e., $\forall d_{k}>0$,
MRPO update will have a higher rate of following the likelihood gradient than Multi-DPO,
which is desirable because we want to fix the likelihood faster to
correct the implicit reward. On the contrary, when the reward estimations
are right, i.e., $\forall d_{k}<0$, MRPO update will have a lower rate of
following the likelihood gradient than Multi-DPO, which may help reduce
over-fitting and stabilize the convergence.

\section{Experimental Results}\label{sec:exp}

In our experiments, we will always refer to the first (initializing) 
reference model as RefM1, the second as RefM2, and so forth. The Base model is the original
LLM that will undergo finetuning on preference data and is initialized as RefM1. Throughout experiments, if
not stated otherwise, Llama (L), Mistral (M), and Qwen (Q) refer to \emph{Llama-2-7b-chat-hf}, \emph{OpenHermes-2.5-Mistral-7B}, and \emph{Qwen1.5-7B-Chat}, respectively.
Unless specified otherwise, we finetune these LLMs using LoRA 4-bit quantization to enable faster training and accommodate our hardware of a single Tesla A100 GPU with 32GB of memory. Further training details
are provided in Appendix \ref{train_detail}.

To assess model performance, we finetune the model with preference
data and evaluate it on preference learning and general language
understanding (GLU) tasks. In preference learning, we measure the
\textbf{Preference Accuracy} in predicting whether two responses, $y_{1}|x$ and
$y_{2}|x$ as chosen (preferred) or rejected (dispreferred). In particular,
following Rafailov et al. (2023), for DPO and Multi-DPO, we use $\beta\log\frac{\pi_{\theta}\left(y|x\right)}{\pi_{ref}\left(y|x\right)}$
as the reward $r(x,y)$ for each response, with the response classified as chosen
or rejected if it has a higher or lower reward, respectively. As
for MRPO, we use $\beta\log\frac{\pi_{\theta}\left(y|x\right)}{\tilde{\pi}_{ref}\left(y|x\right)}$
as the reward. Another preference metric worth considering is the \textbf{Reward Margin}, which quantifies the difference between the chosen and rejected rewards: $r(x,y_w)-r(x,y_l)$. In GLU, we employ the finetuned LLM to generate output
for each input and use the \textbf{GLU Metric} provided by the task. All of these measurements are desirable when they are higher.

\subsection{Performance when Preference Data is Scarce \label{subsec:Performance-when-Preference}}

\textbf{Datasets} In many real-world scenarios, human feedback is
limited. Here, we curate 3 small preference datasets (hundreds to
a few thousand data points) to simulate the scarcity of feedback data.
Each training dataset comprises a random subset from a larger
public preference dataset available in the Hugging Face data repository.
The remaining portions of the dataset will be utilized as testing
data. These datasets are generated with inputs, outputs, and preference
rankings often produced by powerful LLMs like GPT4, making them suitable
for training smaller LLMs such as Llama and Mistral. The datasets
are labeled as S1, S2, and S3, and their details are given in Appendix
Table \ref{tab:Small-datasets:-S1} and Appendix \ref{subsec:Datasets}.

\begin{table}
\begin{centering}
\begin{tabular}{ccccccc}
\hline 
\multirow{2}{*}{Dataset} & \multicolumn{2}{c}{S1} & \multicolumn{2}{c}{S2} & \multicolumn{2}{c}{S3}\tabularnewline
 & L$\leftarrow$M & M$\leftarrow$L & L$\leftarrow$M & M$\leftarrow$L & L$\leftarrow$M & M$\leftarrow$L\tabularnewline
\hline 
DPO & 93.9{\footnotesize{}$\pm$1.4} & \emph{98.7}\emph{\footnotesize{}$\pm$0.7} & 94.4{\footnotesize{}$\pm$2.9} & \emph{96.7}\emph{\footnotesize{}$\pm$1.8} & \emph{54.8}\emph{\footnotesize{}$\pm$5.4} & \emph{52.2}\emph{\footnotesize{}$\pm$3.7}\tabularnewline
Multi-DPO & \emph{95.6}\emph{\footnotesize{}$\pm$1.7} & 98.0{\footnotesize{}$\pm$0.8} & \emph{95.5}\emph{\footnotesize{}$\pm$1.2} & 96.3{\footnotesize{}$\pm$1.2} & 45.6{\footnotesize{}$\pm$4.4} & 51.3{\footnotesize{}$\pm$1.2}\tabularnewline
MRPO (Ours) & \textbf{97.2}\textbf{\footnotesize{}$\pm$1.7} & \textbf{99.5}\textbf{\footnotesize{}$\pm$0.8} & \textbf{97.0}\textbf{\footnotesize{}$\pm$3.2} & \textbf{97.0}\textbf{\footnotesize{}$\pm$1.6} & \textbf{61.3}\textbf{\footnotesize{}$\pm$5.2} & \textbf{56.0}\textbf{\footnotesize{}$\pm$1.9}\tabularnewline
\hline 
\end{tabular}
\par\end{centering}
~

\caption{Final mean$\pm$std. testing accuracy ($\times100$) on small datasets
over 5 runs. Bold denotes the best, statistically different from the others as Cohen effect size > 0.5. Italic denotes the second best.\label{tab:Final-meanstd.-test}}
\end{table}

\textbf{Baselines }Unless stated otherwise, we always compare our MRPO with DPO and Multi-DPO using
the same common hyperparameters such as learning rate ($10^{-5}$), batch size (8),
number of epochs (3), and $\beta=0.1$. For Multi-DPO, we have to
use clipped trust regions to ensure RefM2 is close to RefM1. Otherwise,
the learning will not converge. To make fair comparison, both MRPO
and Multi-DPO use $\epsilon_{max}=0.1$ and incorporate the adaptive
$\epsilon$ ($\mathsection$ \ref{subsec:Optimization-with-Clipped})
and $\alpha$ ($\mathsection$ \ref{subsec:The-Choice-of}) mechanisms.
For Multi-DPO and MRPO, we consider 2 reference models ($K=2$), and examine 2 modes of initialization: (1) L $\leftarrow$ M,
the Base model is initialized as Llama (RefM1) for all baselines,
and the RefM2 is Mistral for MRPO and Multi-DPO; (2) M $\leftarrow$
L, the order is reversed.

\textbf{Results }
Table \ref{tab:Final-meanstd.-test} reports the final preference accuracy on test sets. MRPO
consistently outperforms DPO by a significant margin,
3-7\% and 1-4\% for L $\leftarrow$ M and M $\leftarrow$ L, respectively.
Mode L $\leftarrow$ M observes more improvement because Mistral is
stronger than Llama in these tasks and thus, using Mistral as RefM2
will bring more benefits than using Llama. On the other hand, Multi-DPO
underperforms DPO and MRPO in many cases, indicating that combining multiple
references in a naive way, even when equipped with our CTRO and ARWC, does not yield favorable results.

Appendix Fig. \ref{fig:Chosen/Rejection-prediction-accu}
depicts the testing preference accuracy curves over the training duration for all methods
across 2 initialization modes. These curves demonstrate that all methods
boost the chosen/rejection prediction accuracy of the Base model (the
first evaluation point in each graph is lower than the following points).
Among all, MRPO exhibits early outperformance compared to other baselines
and maintains its superior performance until convergence.

\subsection{Can MRPO Scale to Big Preference Datasets?}

\textbf{Datasets} To assess the scalability of MRPO to real and large
datasets, we utilize three big preference datasets: HelpSteer, Ultrafeedback,
and Nectar (see Appendix Table \ref{tab:Big-datasets:-HelpSteer,}).
Each dataset employs human rankings to assess the outputs generated
by powerful LLMs. 
Finetuning LLMs for just one epoch is sufficient for large datasets to achieve learning
convergence. We use the provided train/test split for HelpSteer and
Ultrafeedback. We randomly allocate 90\% of the data for training
purposes, reserving the remaining 10\% for testing for Nectar.

\textbf{Baselines and Results} In this task, we evaluated MRPO ($K=2$) against
DPO, the top two methods from our prior tests, using the same initialization
approaches detailed earlier. The preference accuracy result, reported in Table \ref{tab:Final-meanstd.-test-1}'s upper row
and Appendix Fig. \ref{fig:3big}, demonstrates that MRPO consistently
surpasses DPO in real-world preference datasets, showing an improvement
gain of approximately 3-5\% and up to 1\% for L$\leftarrow$M and
M$\leftarrow$L modes, respectively. Since Preference Accuracy can sometimes be unclear in demonstrating performance, especially with borderline inputs where the chosen and rejected ground truth may not be entirely accurate, we also examine the Reward Margin of DPO and MRPO on these datasets. The result, displayed in Table \ref{tab:Final-meanstd.-test-1}'s lower row and Appendix Fig. \ref{fig:3bigmargin}, demonstrates MRPO's superior ability to separate chosen and rejected outputs, as evidenced by a significantly higher Reward Margin of 10-20\% compared to DPO. The findings confirm MRPO's capability
to effectively scale with large datasets for preference learning tasks.

\begin{table}
\begin{centering}
\small
\begin{tabular}{cccccccc}

\hline 
\multirow{2}{*}{Metric} & \multirow{2}{*}{Dataset} & \multicolumn{2}{c}{HelpSteer} & \multicolumn{2}{c}{Ultrafeedback} & \multicolumn{2}{c}{Nectar}\tabularnewline
 &  & L$\leftarrow$M & M$\leftarrow$L & L$\leftarrow$M & M$\leftarrow$L & L$\leftarrow$M & M$\leftarrow$L\tabularnewline
\hline 
\multirow{2}{*}{Accuracy} & DPO & 68.9{\footnotesize{}$\pm$0.4} & 70.8{\footnotesize{}$\pm$4.3} & 69.8{\footnotesize{}$\pm$0.9} & 72.0{\footnotesize{}$\pm$1.1} & 75.6{\footnotesize{}$\pm$3.9} & 78.5{\footnotesize{}$\pm$3.6}\tabularnewline
 & MRPO & \textbf{73.6}\textbf{\footnotesize{}$\pm$1.6} & \text{71.6}\text{\footnotesize{}$\pm$5.2} & \textbf{72.9}\textbf{\footnotesize{}$\pm$2.9} & \textbf{73.2}\textbf{\footnotesize{}$\pm$1.9} & \textbf{79.2}\textbf{\footnotesize{}$\pm$1.2} & \text{78.7}\text{\footnotesize{}$\pm$3.0}\tabularnewline
\hline 
\multirow{2}{*}{Margin} & DPO & 0.64{\footnotesize{}$\pm$0.01} & 0.95{\footnotesize{}$\pm$0.20} & 0.70{\footnotesize{}$\pm$0.07} & 1.14{\footnotesize{}$\pm$0.12} & 1.52{\footnotesize{}$\pm$0.08} & 2.65{\footnotesize{}$\pm$0.29}\tabularnewline
 & MRPO & \textbf{0.77}\textbf{\footnotesize{}$\pm$0.07} & \textbf{1.05}\textbf{\footnotesize{}$\pm$0.20} & \textbf{0.82}\textbf{\footnotesize{}$\pm$0.10} & \textbf{1.27}\textbf{\footnotesize{}$\pm$0.26} & \textbf{1.73}\textbf{\footnotesize{}$\pm$0.05} & \textbf{3.13}\textbf{\footnotesize{}$\pm$0.25}\tabularnewline
\hline 
\end{tabular}
\par\end{centering}
~

\caption{Final mean$\pm$std. testing preference accuracy (upper, $\times100$) and reward
margin (lower) on big datasets over 3 runs. Bold is best, statistically different from others
as Cohen effect size > 0.5.\label{tab:Final-meanstd.-test-1}}
\end{table}

\subsection{How Effective is MRPO on General Language Understanding Benchmarks?}

For our evaluation benchmark, we utilized the Huggingface Open LLM
Leaderboard, a standard in the field \cite{beeching2023open}. In
this benchmark, we explore a variety of datasets, collectively covering
tasks such as math (GSM8k) multi-task language understanding (MMLU),
human falsehood understanding (TruthfulQA), and commonsense reasoning
(Arc, HellaSwag, Winogrande). The evaluation process presents the
language models with few-shot in-context examples and questions. We
apply the standard evaluation protocol to evaluate and report average scores (GLU Metrics) across all datasets. 

In this experiment, we consider the third reference model RefM3 as Qwen to verify the scalability of our method to $K=3$ on the standard GLU benchmark. We examine the following initialization modes: (1) L $\leftarrow$ M, Q where Mistral and Qwen are additional reference models for Base model Llama, (2) M $\leftarrow$L, Q where Llama and Qwen are additional reference models for Base model Mistral, and (3) Q $\leftarrow$M, L where Mistral and Llama are additional reference models for Base model Qwen. Following prior practices \cite{chen2024self}, we adopt \textit{full finetuning} to finetune the Base models on Ultrafeedback dataset using DPO and MRPO and evaluate the LLMs using
Language Model Evaluation Harness library \cite{gao10256836framework}. We
report all results in Table \ref{tab:fullall}. MRPO leads to a notable enhancement in the Base model of 3.5\%, 1.4\%, and 0.8\%, surpassing DPO with an average improvement of 1.1\%, 1.0\% and 1.3\% for initialization (1), (2) and (3), respectively. Notably, there are several cases in which MRPO can outperform DPO by a huge margin, such as 6.8\% in GSM8K (M $\leftarrow$L, Q), 5.8\% in TruthfulQA (L $\leftarrow$M, Q) and 5\% in GSM8K (Q $\leftarrow$M, L). 

We also explore MRPO ($K=2$, L $\leftarrow$M and M $\leftarrow$L) and observe that this variant consistently surpasses DPO in performance, albeit falling short of MRPO ($K=3$), suggesting the advantages of incorporating more reference models.  Full results can be found in Appendix Table \ref{tab:fullall2}. 





\begin{table}
\begin{centering}
\small
\begin{tabular}{cccccccc}
\hline 
Dataset & GSM8K & TruthfulQA & HellaSwag & MMLU & Arc-easy & Winograde & Avg.\tabularnewline
\hline 
Base (Mistral) & 49.6 & 44.5 & 62.8 & 60.8 & 83.5 & \textbf{74.4} & 62.6\tabularnewline
DPO & 53.7 & \textbf{53.3} & \textbf{66.6} & 60.1 & 82.7 &73.6 & 65.0 \tabularnewline
MRPO (M$\leftarrow$L,Q) & \textbf{60.5} & 51.43 & 65.83 & \textbf{61.5} & \textbf{84.0} & 73.4 & \textbf{66.1}\tabularnewline
\hline
Base (LLama) & 23.9 & 37.8 & 57.8 & 46.4 & 60.8 & 66.4 & 51.0\tabularnewline
DPO & 22.0 & 39.5 & \textbf{59.4} & 46.3	 & 73.5 & \textbf{67.3} & 51.4 \tabularnewline
MRPO (L$\leftarrow$M,Q) & \textbf{24.3} & \textbf{45.3} & 57.9 & \textbf{46.4} & \textbf{74.1} & 66.7 & \textbf{52.4} \tabularnewline
\hline
Base (Qwen) &   21.1	 & 53.6	  & 58.8 & 	60.1 & 	68.3 & 65.2	 & 54.5 \tabularnewline
DPO  &  18.7 & 	54.8	 & \textbf{60.4}	 & \textbf{60.3}	 & 65.2	 & 64.3	 & 54.0 \tabularnewline
MRPO (Q$\leftarrow$M,L)  & \textbf{23.7} & 	\textbf{54.9} & 	59.0 & 	60.1 & 	\textbf{68.4} & 	\textbf{65.4}  & \textbf{55.3} \tabularnewline
\hline 
\end{tabular}
\par\end{centering}
~

\caption{M$\leftarrow$L: test performance ($\times100$) across HuggingFace
Open LLM Leaderboard datasets. Bold denotes best.\label{tab:fullall}}
\end{table}

\subsection{Distillation from Strong to Weak LLMs}

In this experiment, we assess the advantage of MRPO in transferring
likelihood predictions from larger LLMs to smaller ones. This has
potential applications in improving finetuning small LLMs on mobile
devices. It is worth mentioning that MRPO allows the precomputation
of log probabilities from large LLMs, ensuring efficient training
on low-resource devices. In particular, we choose TinyLlama 
(1.1B parameters) and Mistral (7B parameters) as RefM1 and RefM2, respectively.
We finetune TinyLlama on the Ultrafeedback dataset using a process
akin to the one described in the preceding section and report the
Open LLM Leaderboard results in Table \ref{tab:tiny}. Overall, MRPO keeps outperforming DPO and Base models by 0.2\% and 0.5\%, respectively. While the gain is modest compared to previous results, our findings confirm the benefit of applying MRPO to weak LLMs, especially when the cost of training can be almost similar to that of DPO.
\begin{table}
\begin{centering}
\small
\begin{tabular}{cccccccc}
\hline 
Dataset & GSM8K & TruthfulQA & HellaSwag & MMLU & Arc-easy & Winograde & Avg.\tabularnewline
\hline 
Base (TinyLlama) & \textbf{2.5} & 32.9 & 45.7 & 24.9 & 66.8 & 61.1 & 38.9\tabularnewline
DPO & 2.1 & \textbf{33.1} & 46.0 & 25.4 & \textbf{67.0} & \textbf{61.6} & 39.2\tabularnewline
MRPO (Ours) & 2.1 & 33.0 & \textbf{47.0} & \textbf{25.8} & \textbf{67.0} & \textbf{61.6} & \textbf{39.4}\tabularnewline
\hline 
\end{tabular}
\par\end{centering}
~

\caption{TinyLlama$\leftarrow$Mistral: test performance ($\times100$) across
HuggingFace Open LLM Leaderboard datasets. Bold denotes best.\label{tab:tiny}}
\end{table}

\subsection{Ablation Study}

\subsubsection{The Importance of Clipped Trust-Region Optimization (CTRO)}

\textbf{Can MRPO work without CTRO} ($\mathsection$
\ref{subsec:Optimization-with-Clipped}) To investigate the question,
we utilize the small yet relatively challenging dataset S3 and conduct experiments with
various $\epsilon_{max}$ values: $\epsilon_{max}=0$ (MRPO equals
DPO), $\epsilon_{max}=10^{6}$(No clip), and $\epsilon_{max}=\left\{ 1,0.1,0.01\right\} $. We also examine different reference models:
(i) We employ RefM2 (Llama supervised tuning on Alpaca dataset) as
a finetuned model of RefM1 (Llama), aiming to ensure that RefM2 is
closely related to RefM1 (same family); and (ii) RefM2 (Mistral) is from a different
family from RefM1 (LLama), which indicates a larger mismatch between the
two reference models. Here, different families mean that LLMs can vary in
architecture and/or be pretrained on distinct corpora, where the log probability of these LLMs for the same input
can differ by hundreds of units. The final preference
accuracy is reported in Table \ref{tab:Clipped-Trust-Region-Optimizatio}. 

We observe that for setting (ii) when $\epsilon_{max}\geq1$, the
training loss can escalate significantly, reaching values as high
as 10, and occasionally even infinity, which highlights the instability
of training in the absence of CTRO. This is evident in the poor performance
of $\epsilon_{max}=\left\{ 10^{6},1\right\} $. Utilizing CTRO with
small $\epsilon_{max}$ results in more stable training. However,
excessive constraint, where $\epsilon_{max}$ is too small, can lead
to nearly identical performance compared to DPO. In setting (i), even
with big $\epsilon_{max}$, the training is stable. However, the performance
is not as good as setting (ii). In particular, with the best $\epsilon_{max}=0.1$,
in setting (i), MRPO only achieves a 2\% improvement over DPO, whereas
in setting (ii), the improvement gap widens to 7\%. This discrepancy
is understandable because without a diverse reference source, RefM2
may not exhibit significant advantages over RefM1, thereby limiting
the extent of improvement. Therefore, we conclude that it is more
advantageous to leverage diverse reference models, and employing CTRO
is required to ensure training stability in this case.

\textbf{Is adaptive $\epsilon$ necessary? }We conduct
more experiments with fixed $\epsilon=0.1$ and adaptive $\epsilon_{max}=0.1$
on S1, S2 and S3 using Llama and Mistral for RefM1 and RefM2, receptively.
The results depicted in Appendix Fig. \ref{fig:Analysis-on-} (top)
illustrate that fixed $\epsilon$ is still better than DPO, and adaptive
$\epsilon$ outperforms significantly fixed $\epsilon$ across all
datasets, emphasizing the importance of this mechanism in MRPO.

\subsubsection{Analyzing Reference Weighting Coefficients}

In this section, using adaptive $\epsilon_{max}=0.1$, we compare
the adaptive $\alpha$ (ARWC) mechanism proposed in $\mathsection$
\ref{subsec:The-Choice-of} with different fixed values of $\alpha=\left\{ 0.1,0.5,0.9\right\} $
on S1, S2 and S3 using Llama and Mistral for RefM1 and RefM2, receptively.
As shown in Appendix Fig. \ref{fig:Analysis-on-} (bottom), adaptive
$\alpha$ demonstrates competitive performance, either outperforming
or closely matching the performance of the best fixed $\alpha$ across
all datasets. Given the minor discrepancies observed and the associated
cost of hyperparameter tuning, we have opted to utilize adaptive $\alpha$
for all other experiments.

\begin{table}
\begin{centering}
\small
\begin{tabular}{ccccccccccc}
\hline 
\multirow{2}{*}{Setting} & \multicolumn{2}{c}{$\epsilon_{max}=0$ (DPO)} & \multicolumn{2}{c}{$\epsilon_{max}=10^{6}$ (No clip)} & \multicolumn{2}{c}{$\epsilon_{max}=1$} & \multicolumn{2}{c}{$\epsilon_{max}=0.1$} & \multicolumn{2}{c}{$\epsilon_{max}=0.01$}\tabularnewline
 & (i) & (ii) & (i) & (ii) & (i) & (ii) & (i) & (ii) & (i) & (ii)\tabularnewline
\hline 
Test Acc. & 0.54 & 0.54 & 0.51 & 0.45 & 0.53 & 0.49 & 0.56 & \textbf{0.61} & 0.54 & 0.55\tabularnewline
\hline 
\end{tabular}\\
~
\par\end{centering}
\caption{Clipped Trust-Region Optimization impact on S3. In setting (i), two
reference models belong to the same family but differ in the finetuning
dataset, whereas in setting (ii), they are from different families
of LLMs. The reported numbers are the mean accuracy over 5 runs. \label{tab:Clipped-Trust-Region-Optimizatio}}

\end{table}

\section{Related work}

The integration of human input has been instrumental in advancing
the performance of Large Language Models (LLMs) in diverse domains,
such as question answering \cite{nakano2021webgpt}, document summarization
\cite{stiennon2020learning}, and dialog applications \cite{thoppilan2022lamda}.
Traditionally, instruction finetuning (IFT) and reinforcement learning
from human feedback (RLHF) framework \cite{christiano2017deep,ouyang2022training,lee2023rlaif}
has employed RL to align Large Language Models (LLMs). They RLH objective
is to maximize a reward score derived from human preferences (chosen
or rejected) while simultaneously minimizing the disparity between
the new and initial policy. Recently, there has been a significant
move towards closed-form losses, exemplified by DPO \cite{rafailov2023direct},
which directly finetunes LLM on offline preference datasets, consolidating
RLHF's reward learning and policy adjustments into a single stage.
This ``direct'' approach is favored over RLHF due to their maximum-likelihood losses, demonstrating superior speed and stability than RL pipeline.

Other direct preference optimization methods \cite{zhao2023slic,azar2023general}
formulate various adaptations of closed-form losses to attain the
RLHF objective. Recent advancements have expanded beyond the traditional
binary preference data, focusing on novel human preference models
like Kahneman-Tversky value functions \cite{ethayarajh2024kto}. All
these methods adhere to the fundamental framework of RLHF, striving
to fit a preference model while ensuring the updates remain close
to a reference model. In contrast, \emph{our approach is the first direct preference finetuning framework with multiple
reference models}.

Another line of work creates new preference data from LLM's own generated
outputs, typically through self-training paradigms \cite{chen2024self,yuan2024self,pattnaik2024curry},
 employing multiple reference data pairs \cite{pattnaik2024curry}.
Our method is orthogonal to self-playing approaches, featuring a faster
alternative procedure as it does not require additional data generation. In contrast to methods employing multiple rewards
and model merging \cite{jang2023personalized,rame2024rewarded},
our approach does not involve training multiple LLMs. Instead, we
train a single LLM using log-probability outputs of multiple reference
LLMs during training. During testing, there is no need to maintain
multiple models, and our inference cost is the same as using a single
LLM.

\section{Discussion}

In this paper, we present Multi-Reference Preference Optimization
(MRPO), a novel method leveraging multiple reference models to improve
preference learning for Large Language Models (LLMs). We theoretically
derive the objective function for MRPO and conduct experiments with
LLMs like LLama2 and Mistral, demonstrate their enhanced generalization
across six preference datasets and competitive performance in six
downstream natural language processing tasks.

\textbf{Limitations}\label{sec:limit}
Our study is limited by using only up to three reference models of modest size. Future research will explore the scalability of MRPO, examining its performance with larger 
$K$ values and a broader range of LLM sizes across diverse benchmarks. 
The limited improvement gain on small LLMs like TinyLlama presents a challenge for our current approach. Further investigation is needed to enhance MRPO when small LLMs serve as the base model, common on low-resource devices.

\textbf{Broader Impacts}\label{sec:impact}
In this work, we used publicly available datasets for our experiments and did not collect any human or animal data. We aim to enhance the alignment of LLMs with human preferences and values. We believe this goal is genuine and do not foresee any immediate harmful consequences. However, we acknowledge potential issues if our method is used to augment language models to generate hallucinated or negative content. This risk is inherent to any fine-tuning method when the fine-tuned data can be misused, and we will take all possible measures to prevent such misuse from our end.

\bibliographystyle{plainnat}
\bibliography{mrpo}

\cleardoublepage{}

\renewcommand\thesubsection{\Alph{subsection}}

\section*{Appendix}

\subsection{Methodology Details}

\subsubsection{Lower Bound Close-formed Solution\label{subsec:Lower-Bound-Close-formed}}

To begin, we derive the closed-form solution for Eq. \ref{eq:prefm-rlhf-1}
as follows,

\begin{align}
\underset{\pi}{\max} & \underset{y\sim\pi\left(\cdot|x\right)}{\underset{x\sim\rho}{\mathbb{E}}}\left[r\left(x,y\right)\right]-\text{\ensuremath{\beta}}\left(\sum_{k=1}^{K}\alpha_{k}D_{KL}\left(\pi\left\Vert \pi_{ref}^{k}\right.\right)\right)\\
\iff\underset{\pi}{\max} & \mathbb{E}_{x\sim\rho}\mathbb{E}_{y\sim\pi\left(\cdot|x\right)}\left[r\left(x,y\right)-\text{\ensuremath{\beta}}\left(\sum_{k=1}^{K}\alpha_{k}\log\frac{\pi\left(y|x\right)}{\pi_{ref}^{k}\left(y|x\right)}\right)\right]\\
\iff\underset{\pi}{\min} & \mathbb{E}_{x\sim\rho}\mathbb{E}_{y\sim\pi\left(\cdot|x\right)}\left[\left(\sum_{k=1}^{K}\alpha_{k}\log\frac{\pi\left(y|x\right)}{\pi_{ref}^{k}\left(y|x\right)}\right)-\frac{1}{\beta}r\left(x,y\right)\right]\label{eq:min_obj}
\end{align}
To simplify the optimization problem, using Jensen inequality $\sum_{k=1}^{K}\alpha_{k}\log\left(\frac{\pi\left(y|x\right)}{\pi_{ref}^{k}\left(y|x\right)}\right)\leq\log\left(\sum_{k=1}^{K}\alpha_{k}\frac{\pi\left(y|x\right)}{\pi_{ref}^{k}\left(y|x\right)}\right)$,
we propose minimizing the following upper bound of the objective function
in Eq. \ref{eq:min_obj} (equivalent to a lowerbound of the original
RLHF objective):

\begin{align}
\underset{\pi}{\min} & \mathbb{E}_{x\sim\rho}\mathbb{E}_{y\sim\pi\left(\cdot|x\right)}\left[\log\left(\sum_{k=1}^{K}\alpha_{k}\frac{\pi\left(y|x\right)}{\pi_{ref}^{k}\left(y|x\right)}\right)-\frac{1}{\beta}r\left(x,y\right)\right]\\
\iff\underset{\pi}{\min} & \mathbb{E}_{x\sim\rho}\mathbb{E}_{y\sim\pi\left(\cdot|x\right)}\left[\text{\ensuremath{\log}}\left(\pi\left(y|x\right)\sum_{k=1}^{K}\frac{\alpha_{k}}{\pi_{ref}^{k}\left(y|x\right)}\right)-\frac{1}{\beta}r\left(x,y\right)\right]\\
\iff\underset{\pi}{\min} & \mathbb{E}_{x\sim\rho}\mathbb{E}_{y\sim\pi\left(\cdot|x\right)}\left[\text{\ensuremath{\log}}\left(\frac{\pi\left(y|x\right)}{\frac{1}{Z\left(x\right)}\left(\sum_{k=1}^{K}\frac{\alpha_{k}}{\pi_{ref}^{k}\left(y|x\right)}\right)^{-1}\exp\left(\frac{1}{\beta}r\left(x,y\right)\right)}\right)-\log Z\left(x\right)\right]\label{eq:mref_sol}
\end{align}
where we have the partition function

\begin{align}
Z\left(x\right) & =\sum_{y}\left(\sum_{k=1}^{K}\frac{\alpha_{k}}{\pi_{ref}^{k}\left(y|x\right)}\right)^{-1}\exp\left(\frac{1}{\beta}r\left(x,y\right)\right)\\
 & =\sum_{y}\tilde{\pi}_{ref}\left(y|x\right)\exp\left(\frac{1}{\beta}r\left(x,y\right)\right)
\end{align}
where
\begin{align}
\tilde{\pi}_{ref} & =\left(\sum_{k=1}^{K}\frac{\alpha_{k}}{\pi_{ref}^{k}\left(y|x\right)}\right)^{-1}\label{eq:vir_ref}\\
\Longleftrightarrow\frac{\pi_{\theta}\left(y_{w}|x\right)}{\tilde{\pi}_{ref}\left(y_{w}|x\right)} & =\sum_{k=1}^{K}\alpha_{k}\frac{\pi_{\theta}\left(y_{w}|x\right)}{\pi_{ref}^{k}\left(y|x\right)}
\end{align}
We note that $\tilde{\pi}_{ref}>0$ and it is not necessary that $\tilde{\pi}_{ref}$
is a distribution. If we define the distribution:

\[
\pi^{*}\left(y|x\right)=\frac{1}{Z\left(x\right)}\tilde{\pi}{}_{ref}\left(y|x\right)\exp\left(\frac{1}{\beta}r\left(x,y\right)\right)
\]
then the final objective is:

\[
\underset{\pi}{\min}\mathbb{E}_{x\sim\rho}\left[D_{KL}\left(\pi\left(y|x\right)||\pi^{*}\left(y|x\right)\right)-\log Z\left(x\right)\right]
\]
Therefore, the optimal solution for the problem defined in Eq. \ref{eq:mref_sol}
is:

\begin{equation}
\pi\left(y|x\right)=\pi^{*}\left(y|x\right)=\frac{1}{Z\left(x\right)}\tilde{\pi}{}_{ref}\left(y|x\right)\exp\left(\frac{1}{\beta}r\left(x,y\right)\right)
\end{equation}

\subsubsection{MRPO and Multi-DPO Gradient Inequality \label{subsec:MRPO-and-Multi-DPO}}

Assuming the reference policies are constrained to be relatively close
to each other, ensuring that $\left\{ d_{k}=r_{\theta}\left(x,y_{l}|\pi_{ref}^{k}\right)-r_{\theta}\left(x,y_{w}|\pi_{ref}^{k}\right)\right\} _{k=1}^{K}$
share the same sign $\forall k$, we can use Jensen inequality to
show that 

\[
\begin{cases}
\sigma\left(\beta\sum_{k=1}^{K}\alpha_{k}\left(\log\frac{\pi_{\theta}\left(y_{l}|x\right)}{\pi_{ref}^{k}\left(y_{l}|x\right)}-\log\frac{\pi_{\theta}\left(y_{w}|x\right)}{\pi_{ref}^{k}\left(y_{w}|x\right)}\right)\right) & \geq\sum_{k=1}^{K}\alpha_{k}\sigma\left(\beta\log\frac{\pi_{\theta}\left(y_{l}|x\right)}{\pi_{ref}^{k}\left(y_{l}|x\right)}-\beta\log\frac{\pi_{\theta}\left(y_{w}|x\right)}{\pi_{ref}^{k}\left(y_{w}|x\right)}\right)\,\,\,\forall d_{k}\geq0\\
\sigma\left(\beta\sum_{k=1}^{K}\alpha_{k}\left(\log\frac{\pi_{\theta}\left(y_{l}|x\right)}{\pi_{ref}^{k}\left(y_{l}|x\right)}-\log\frac{\pi_{\theta}\left(y_{w}|x\right)}{\pi_{ref}^{k}\left(y_{w}|x\right)}\right)\right) & \leq\sum_{k=1}^{K}\alpha_{k}\sigma\left(\beta\log\frac{\pi_{\theta}\left(y_{l}|x\right)}{\pi_{ref}^{k}\left(y_{l}|x\right)}-\beta\log\frac{\pi_{\theta}\left(y_{w}|x\right)}{\pi_{ref}^{k}\left(y_{w}|x\right)}\right)\,\,\,\forall d_{k}\leq0
\end{cases}
\]

\[
\Rightarrow\begin{cases}
\sigma\left(\beta\sum_{k=1}^{K}\alpha_{k}\left(\log\frac{\pi_{\theta}\left(y_{l}|x\right)}{\pi_{ref}^{k}\left(y_{l}|x\right)}+\log\frac{\pi_{ref}^{k}\left(y_{w}|x\right)}{\pi_{\theta}\left(y_{w}|x\right)}\right)\right) & \geq\sum_{k=1}^{K}\alpha_{k}\sigma\left(\beta\log\frac{\pi_{\theta}\left(y_{l}|x\right)}{\pi_{ref}^{k}\left(y_{l}|x\right)}-\beta\log\frac{\pi_{\theta}\left(y_{w}|x\right)}{\pi_{ref}^{k}\left(y_{w}|x\right)}\right)\\
\forall d_{k}\geq0\\
\sigma\left(\beta\sum_{k=1}^{K}\alpha_{k}\left(-\log\frac{\pi_{ref}^{k}\left(y_{l}|x\right)}{\pi_{\theta}\left(y_{l}|x\right)}-\log\frac{\pi_{\theta}\left(y_{w}|x\right)}{\pi_{ref}^{k}\left(y_{w}|x\right)}\right)\right) & \leq\sum_{k=1}^{K}\alpha_{k}\sigma\left(\beta\log\frac{\pi_{\theta}\left(y_{l}|x\right)}{\pi_{ref}^{k}\left(y_{l}|x\right)}-\beta\log\frac{\pi_{\theta}\left(y_{w}|x\right)}{\pi_{ref}^{k}\left(y_{w}|x\right)}\right)\\
\forall d_{k}\leq0
\end{cases}
\]

\[
\Rightarrow\begin{cases}
\sigma\left(\beta\log\sum_{k=1}^{K}\alpha_{k}\frac{\pi_{\theta}\left(y_{l}|x\right)}{\pi_{ref}^{k}\left(y_{l}|x\right)}+\beta\log\sum_{k=1}^{K}\alpha_{k}\frac{\pi_{ref}^{k}\left(y_{w}|x\right)}{\pi_{\theta}\left(y_{w}|x\right)}\right) & \geq\sum_{k=1}^{K}\alpha_{k}\sigma\left(\beta\log\frac{\pi_{\theta}\left(y_{l}|x\right)}{\pi_{ref}^{k}\left(y_{l}|x\right)}-\beta\log\frac{\pi_{\theta}\left(y_{w}|x\right)}{\pi_{ref}^{k}\left(y_{w}|x\right)}\right)\\
\forall d_{k}\geq0\\
\sigma\left(-\beta\log\sum_{k=1}^{K}\alpha_{k}\frac{\pi_{ref}^{k}\left(y_{l}|x\right)}{\pi_{\theta}\left(y_{l}|x\right)}-\beta\log\sum_{k=1}^{K}\alpha_{k}\frac{\pi_{\theta}\left(y_{w}|x\right)}{\pi_{ref}^{k}\left(y_{w}|x\right)}\right) & \leq\sum_{k=1}^{K}\alpha_{k}\sigma\left(\beta\log\frac{\pi_{\theta}\left(y_{l}|x\right)}{\pi_{ref}^{k}\left(y_{l}|x\right)}-\beta\log\frac{\pi_{\theta}\left(y_{w}|x\right)}{\pi_{ref}^{k}\left(y_{w}|x\right)}\right)\\
\forall d_{k}\leq0
\end{cases}
\]

\[
\Rightarrow\begin{cases}
\sigma\left(\beta\log\sum_{k=1}^{K}\alpha_{k}\frac{\pi_{\theta}\left(y_{l}|x\right)}{\pi_{ref}^{k}\left(y_{l}|x\right)}+\beta\log\sum_{k=1}^{K}\alpha_{k}\frac{\pi_{ref}^{k}\left(y_{w}|x\right)}{\pi_{\theta}\left(y_{w}|x\right)}\right) & \geq\sum_{k=1}^{K}\alpha_{k}\sigma\left(\beta\log\frac{\pi_{\theta}\left(y_{l}|x\right)}{\pi_{ref}^{k}\left(y_{l}|x\right)}-\beta\log\frac{\pi_{\theta}\left(y_{w}|x\right)}{\pi_{ref}^{k}\left(y_{w}|x\right)}\right)\\
\forall d_{k}\geq0\\
\sigma\left(-\beta\log\sum_{k=1}^{K}\alpha_{k}\frac{\pi_{ref}^{k}\left(y_{l}|x\right)}{\pi_{\theta}\left(y_{l}|x\right)}-\beta\log\sum_{k=1}^{K}\alpha_{k}\frac{\pi_{\theta}\left(y_{w}|x\right)}{\pi_{ref}^{k}\left(y_{w}|x\right)}\right) & \leq\sum_{k=1}^{K}\alpha_{k}\sigma\left(\beta\log\frac{\pi_{\theta}\left(y_{l}|x\right)}{\pi_{ref}^{k}\left(y_{l}|x\right)}-\beta\log\frac{\pi_{\theta}\left(y_{w}|x\right)}{\pi_{ref}^{k}\left(y_{w}|x\right)}\right)\\
\forall d_{k}\leq0
\end{cases}
\]
Applying Jensen inequality again with sigmoid monotonic property,
we have:

\[
\begin{cases}
\sigma\left(\beta\log\sum_{k=1}^{K}\alpha_{k}\frac{\pi_{\theta}\left(y_{l}|x\right)}{\pi_{ref}^{k}\left(y_{l}|x\right)}-\beta\log\sum_{k=1}^{K}\frac{\pi_{\theta}\left(y_{w}|x\right)}{\alpha_{k}\pi_{ref}^{k}\left(y_{w}|x\right)}\right) & \geq\sum_{k=1}^{K}\alpha_{k}\sigma\left(\beta\log\frac{\pi_{\theta}\left(y_{l}|x\right)}{\pi_{ref}^{k}\left(y_{l}|x\right)}-\beta\log\frac{\pi_{\theta}\left(y_{w}|x\right)}{\pi_{ref}^{k}\left(y_{w}|x\right)}\right)\\
\forall d_{k}\geq0\\
\sigma\left(\beta\log\sum_{k=1}^{K}\frac{\pi_{\theta}\left(y_{l}|x\right)}{\alpha_{k}\pi_{ref}^{k}\left(y_{l}|x\right)}-\beta\log\sum_{k=1}^{K}\alpha_{k}\frac{\pi_{\theta}\left(y_{w}|x\right)}{\pi_{ref}^{k}\left(y_{w}|x\right)}\right) & \leq\sum_{k=1}^{K}\alpha_{k}\sigma\left(\beta\log\frac{\pi_{\theta}\left(y_{l}|x\right)}{\pi_{ref}^{k}\left(y_{l}|x\right)}-\beta\log\frac{\pi_{\theta}\left(y_{w}|x\right)}{\pi_{ref}^{k}\left(y_{w}|x\right)}\right)\\
\forall d_{k}\leq0
\end{cases}
\]
Now we use log monotonic property and the fact that $1/\alpha_{k}\geq\alpha_{k}$
for $0<\alpha_{k}\leq1$ and assume $0<\alpha_{k}\forall k$, we have:

\[
\begin{cases}
\sigma\left(\beta\log\sum_{k=1}^{K}\alpha_{k}\frac{\pi_{\theta}\left(y_{l}|x\right)}{\pi_{ref}^{k}\left(y_{l}|x\right)}-\beta\log\sum_{k=1}^{K}\alpha_{k}\frac{\pi_{\theta}\left(y_{w}|x\right)}{\pi_{ref}^{k}\left(y_{w}|x\right)}\right) & \geq\sum_{k=1}^{K}\alpha_{k}\sigma\left(\beta\log\frac{\pi_{\theta}\left(y_{l}|x\right)}{\pi_{ref}^{k}\left(y_{l}|x\right)}-\beta\log\frac{\pi_{\theta}\left(y_{w}|x\right)}{\pi_{ref}^{k}\left(y_{w}|x\right)}\right)\\
\forall d_{k}\geq0\\
\sigma\left(\beta\log\sum_{k=1}^{K}\alpha_{k}\frac{\pi_{\theta}\left(y_{l}|x\right)}{\pi_{ref}^{k}\left(y_{l}|x\right)}-\beta\log\sum_{k=1}^{K}\alpha_{k}\frac{\pi_{\theta}\left(y_{w}|x\right)}{\pi_{ref}^{k}\left(y_{w}|x\right)}\right) & \leq\sum_{k=1}^{K}\alpha_{k}\sigma\left(\beta\log\frac{\pi_{\theta}\left(y_{l}|x\right)}{\pi_{ref}^{k}\left(y_{l}|x\right)}-\beta\log\frac{\pi_{\theta}\left(y_{w}|x\right)}{\pi_{ref}^{k}\left(y_{w}|x\right)}\right)\\
\forall d_{k}\leq0
\end{cases}
\]
Plugging Eq. \ref{eq:vir_ref} yields:

\[
\begin{cases}
\sigma\left(\beta\log\frac{\pi_{\theta}\left(y_{l}|x\right)}{\tilde{\pi}_{ref}\left(y_{l}|x\right)}-\beta\log\frac{\pi_{\theta}\left(y_{w}|x\right)}{\tilde{\pi}_{ref}\left(y_{w}|x\right)}\right) & \geq\sum_{k=1}^{K}\alpha_{k}\sigma\left(\beta\log\frac{\pi_{\theta}\left(y_{l}|x\right)}{\pi_{ref}^{k}\left(y_{l}|x\right)}-\beta\log\frac{\pi_{\theta}\left(y_{w}|x\right)}{\pi_{ref}^{k}\left(y_{w}|x\right)}\right)\,\,\,\forall d_{k}\geq0\\
\sigma\left(\beta\log\frac{\pi_{\theta}\left(y_{l}|x\right)}{\tilde{\pi}_{ref}\left(y_{l}|x\right)}-\beta\log\frac{\pi_{\theta}\left(y_{w}|x\right)}{\tilde{\pi}_{ref}\left(y_{w}|x\right)}\right) & \leq\sum_{k=1}^{K}\alpha_{k}\sigma\left(\beta\log\frac{\pi_{\theta}\left(y_{l}|x\right)}{\pi_{ref}^{k}\left(y_{l}|x\right)}-\beta\log\frac{\pi_{\theta}\left(y_{w}|x\right)}{\pi_{ref}^{k}\left(y_{w}|x\right)}\right)\,\,\,\forall d_{k}\leq0
\end{cases}
\]
Note that equality cannot be achieved in this case due to $1/\alpha_{k}=\alpha_{k}$
only when $\forall\alpha_{k}=1$. When there is one $\alpha_{k}=0$,
we can omit the corresponding $\pi_{ref}^{k}$ and follow the same
derivation above to prove the inequality. 

\subsection{Experimental Details}\label{sec:exp_detail}

\subsubsection{Training Details}\label{train_detail}

For experiments using LoRA finetuning, we use LoRA with the following
configuration:
\begin{itemize}
\item Rank: 8
\item $\alpha$: 8
\item Quantization: 4 bit
\item Target layers:
\begin{itemize}
\item 7B LLMs: q\_proj,v\_proj
\item 1B LLMs: q\_proj,k\_proj,v\_proj,o\_proj,gate\_proj,up\_proj,down\_proj
\end{itemize}
\end{itemize}

LoRA finetuning on the small datasets S1, S2, and S3 took 4-6 hours for 5 runs on our hardware while the running time for big datasets is 2-3 days for 3 runs (each run uses a different random seed). Full finetuning on Ultrafeedback took 2-3 days.

We implemented MRPO and other baselines using a public (Pytorch) code repository \url{https://github.com/hiyouga/LLaMA-Factory} (Apache-2.0 license). All the experiments are executed using this code base. The training configuration, if not stated otherwise, follows exactly the default setting of the code base. We only adjust the batch size to suit our hardware. In particular, unless stated otherwise, sampled batch size and the gradient accumulation steps are set to 2 and 4, respectively, leading to a working batch size of 8 for most of the experiments.  
All the pretrained LLMs are downloaded from \url{https://huggingface.co/} and have Apache-2.0 license. The links to download are:
\begin{itemize}
\item \emph{Llama-2-7b-chat-hf}: \url{https://huggingface.co/meta-llama/Llama-2-7b-chat-hf}
\item \emph{OpenHermes-2.5-Mistral-7B}: \url{https://huggingface.co/teknium/OpenHermes-2.5-Mistral-7B}
\item \emph{Qwen1.5-7B-Chat}: \url{https://huggingface.co/Qwen/Qwen1.5-7B-Chat}
\item \emph{TinyLlama}: \url{https://huggingface.co/TinyLlama/TinyLlama-1.1B-Chat-v1.0}
\end{itemize}
\subsubsection{Datasets and Additional Results\label{subsec:Datasets}}

\begin{table}
\begin{centering}
\begin{tabular}{cccc}
\hline 
Dataset & Train/Test size & Topic & Source: \url{https://huggingface.co/datasets}\tabularnewline
\hline 
S1 & 364/3276 & General & \texttt{wangrongsheng/comparison\_gpt4\_data\_en}\tabularnewline
S2 & 502/502 & AI/ML & \texttt{NeuralNovel/Neural-DPO}\tabularnewline
S3 & 2178/242 & Math & \texttt{argilla/distilabel-math-preference-dpo}\tabularnewline
\hline 
\end{tabular}
\par\end{centering}
~

\caption{Small datasets: S1, S2, and S3.\label{tab:Small-datasets:-S1}}
\end{table}

Below are some examples extracted from the small datasets:

\noindent\shadowbox{\begin{minipage}[t]{1\columnwidth - 2\fboxsep - 2\fboxrule - \shadowsize}%
\texttt{\emph{S1}}

\texttt{\uline{Input:}}\texttt{ What are the three primary colors?}

\texttt{\uline{Chosen Output:}}\texttt{ The three primary colors
are red, blue, and yellow. These colors are called primary because
they cannot be created by mixing other colors and all other colors
can be made by combining them in various proportions. In the additive
color system, used for light, the primary colors are red, green, and
blue (RGB). }

\texttt{\uline{Rejected Output:}}\texttt{ Red, Yellow, and Green.}%
\end{minipage}}

\noindent\shadowbox{\begin{minipage}[t]{1\columnwidth - 2\fboxsep - 2\fboxrule - \shadowsize}%
\texttt{\emph{S2}}

\texttt{\uline{Input:}}\texttt{ You are an A.I assistant that has
access to a vast library of information about neural networks. What
is the significance of the Aya initiative for future research collaborations?}

\texttt{\uline{Chosen Output:}}\texttt{ The Aya initiative serves
as a valuable framework for future research collaborations that aim
to bridge gaps in resources.}

\texttt{\uline{Rejected Output:}}\texttt{ The Aya initiative has
no significance for future research collaborations.}%
\end{minipage}}

\noindent\shadowbox{\begin{minipage}[t]{1\columnwidth - 2\fboxsep - 2\fboxrule - \shadowsize}%
\texttt{\emph{S3}}

\texttt{\uline{Input:}}\texttt{ How can I simplify the algebraic
expression `(3x\textasciicircum 2 - 4y\textasciicircum 3) / (2x)`?.\textbackslash n
Take a deep breath, think step by step, and give an accurate response}

\texttt{\uline{Chosen Output:}}\texttt{ To simplify the algebraic
expression `(3x\textasciicircum 2 - 4y\textasciicircum 3) / (2x)`,
we can follow a few steps:\textbackslash n\textbackslash nStep 1: Distribute
the division symbol by multiplying the expression by the reciprocal
of the denominator. The reciprocal of `2x` is `1/(2x)`, so the expression
becomes `(3x\textasciicircum 2 - 4y\textasciicircum 3) {*} (1/(2x))`.\textbackslash n\textbackslash nStep
2: Simplify within the parentheses by dividing each term separately. \textbackslash n\textbackslash n-
For the first term, `3x\textasciicircum 2`, divide `3x\textasciicircum 2`
by `2x`. This gives us `(3x\textasciicircum 2) / (2x) = (3/2) {*}
(x\textasciicircum 2 / x) = (3/2) {*} x`.\textbackslash n\textbackslash n-
For the second term, `-4y\textasciicircum 3`, divide `-4y\textasciicircum 3`
by `2x`. This gives us `(-4y\textasciicircum 3) / (2x) = (-2) {*}
(y\textasciicircum 3 / x)`.\textbackslash n\textbackslash nStep
3: Combine the simplified terms from Step 2. The expression now becomes
`(3/2) {*} x - 2 {*} (y\textasciicircum 3 / x)`.\textbackslash n\textbackslash nSo,
the simplified form of the algebraic expression `(3x\textasciicircum 2
- 4y\textasciicircum 3) / (2x)` is `(3/2) {*} x - 2 {*} (y\textasciicircum 3
/ x)`.}

\texttt{\uline{Rejected Output:}}\texttt{ To simplify the algebraic
expression `(3x\textasciicircum 2 - 4y\textasciicircum 3) / (2x)`,
you can divide each term in the numerator by the denominator.\textbackslash n\textbackslash nFirst,
let's divide `3x\textasciicircum 2` by `2x`. Since both terms have
a common factor of `x`, we can simplify this expression to `3x`.\textbackslash n\textbackslash nNext,
we divide `-4y\textasciicircum 3` by `2x`. We can simplify this expression
by dividing each term separately. \textbackslash n\textbackslash nDividing
`-4` by `2` gives `-2`. Then, dividing `y\textasciicircum 3` by `x`
gives `y\textasciicircum 3/x`.\textbackslash n\textbackslash nSo,
the simplified form of `(3x\textasciicircum 2 - 4y\textasciicircum 3)
/ (2x)` is `3x - 2y\textasciicircum 3/x`.}%
\end{minipage}}

\begin{table}
\begin{centering}
\begin{tabular}{ccc}
\hline 
Dataset & Train/Test size & Source: \url{https://huggingface.co/datasets}\tabularnewline
\hline 
HelpSteer & 8.1K/418 & \texttt{sablo/HelpSteer\_binarized}\tabularnewline
Ultrafeedback & 61K/2K & \texttt{HuggingFaceH4/ultrafeedback\_binarized}\tabularnewline
Nectar & 164K/19K & \texttt{HongchengGao/Nectar\_binarized}\tabularnewline
\hline 
\end{tabular}
\par\end{centering}
~

\caption{Big datasets: HelpSteer, Ultrafeedback, and Nectar.\label{tab:Big-datasets:-HelpSteer,}}
\end{table}

Below are some examples extracted from the big datasets

\noindent\shadowbox{\begin{minipage}[t]{1\columnwidth - 2\fboxsep - 2\fboxrule - \shadowsize}%
\texttt{\emph{HelpSteer}}

\texttt{\uline{Input:}}\texttt{ A group of people have been wrongfully
convicted of a crime that they did not commit. Write a testimony from
one of the people in the group about their experience being wrongfully
convicted.}

\texttt{\uline{Chosen Output:}}\texttt{ My name is {[}Name{]},
and I am one of the wrongfully convicted. I was wrongfully convicted
for a crime that I did not commit, and I want to share my experience
with you.\textbackslash n\textbackslash nIt all started when I was
arrested and charged with a crime that I didn't commit. I was innocent,
but the evidence against me was overwhelming. I was interrogated for
hours, and I eventually confessed to a crime that I didn't commit.\textbackslash n\textbackslash nThe
trial was a circus. The prosecution presented a case that was full
of lies and deceit. The jury didn't believe me, and I was found guilty. I
was sentenced to life in prison.\textbackslash n\textbackslash nFor
years, I struggled to make sense of what happened to me. I was lost
and confused, and I felt like the justice system had failed me. But
then, I found out about {[}Organization{]}, a group of people who
are dedicated to helping those who have been wrongfully convicted.\textbackslash n\textbackslash nI
reached out to {[}Organization{]}, and they started working on my
case. They investigated the evidence, and they found new witnesses
who came forward with new information. They uncovered evidence that
proved my innocence, and after years of fighting, I was finally freed.\textbackslash n\textbackslash nBeing
wrongfully convicted was one of the worst experiences of my life. But
thanks to {[}Organization{]}, I was able to clear my name and start
a new chapter. I want to share my story with you to raise awareness
about the importance of fair trials and the need to protect the rights
of the innocent.\textbackslash n\textbackslash nMy name is {[}Name{]},
and I am one of the wrongfully convicted. But now, I am free.}

\texttt{\uline{Rejected Output:}}\texttt{ Hi there, I am an AI
assistant that can answer questions about a reference text. You can
ask me about the reference text and I will provide an answer. Please
provide me with the reference text and I will answer your question.}%
\end{minipage}}

\noindent\shadowbox{\begin{minipage}[t]{1\columnwidth - 2\fboxsep - 2\fboxrule - \shadowsize}%
\texttt{\emph{Ultrafeedback}}

\texttt{\uline{Input:}}\texttt{ how can i develop a habit of drawing
daily}

\texttt{\uline{Chosen Output:}}\texttt{ Developing a daily habit
of drawing can be challenging but with consistent practice and a few
tips, it can become an enjoyable and rewarding part of your daily
routine. Here are some strategies to help you develop the habit of
drawing daily:\textbackslash n\textbackslash n1. Set a specific
time: Allocate a specific time of the day to draw. It could be in
the morning, afternoon, or evening. Make drawing a part of your daily
routine.\textbackslash n2. Set a specific duration: Determine the
amount of time you want to spend on drawing each day. It can be as
little as 10 minutes or as long as an hour. Be consistent with the
duration to help build the habit.\textbackslash n3. Start small and
simple: Don't try to create a masterpiece every day, start with simple
and easy-to-do sketches. Focus on improving your skills gradually.\textbackslash n4. Use
a variety of tools and mediums: Experiment with different tools like
pencils, pens, markers, and different mediums like paper, canvas,
or digital apps to keep your drawing practice interesting and engaging.\textbackslash n5. Take
breaks and rest: Taking breaks and resting after some time of drawing
can help you avoid burnout and stay motivated.\textbackslash n6. Challenge
yourself: Set challenges like drawing objects from memory or a specific
subject to improve your skills and keep your drawing practice interesting.\textbackslash n7. Track
your progress: Keep a record of your daily drawing practice and track
your progress. This can be a source of motivation and help you see
how far you've come.\textbackslash n\textbackslash nRemember, developing
a habit takes time and patience. Stay consistent with your drawing
practice, be flexible and open to trying new things, and with time,
you'll develop a habit of daily drawing that brings you joy and satisfaction.}

\texttt{\uline{Rejected Output:}}\texttt{ As an AI language model,
I cannot personally develop habits for you. But, here are some tips
for developing a habit of drawing daily:\textbackslash n\textbackslash n1. Start
small: Start with simple drawings or doodles and gradually increase
the complexity of your drawings.\textbackslash n\textbackslash n2. Set
a schedule: Set a specific time every day to draw, and try to stick
to it as much as possible.\textbackslash n\textbackslash n3. Make
it fun: Don't pressure yourself to create masterpieces every time
you draw. Make it a relaxing and enjoyable experience.\textbackslash n\textbackslash n4. Use
resources: There are many drawing tutorials available online. Use
resources like YouTube or online drawing courses to help you improve
your skills.\textbackslash n\textbackslash n5. Surround yourself
with inspiration: Expose yourself to a variety of art forms, such
as paintings, illustrations, and photographs, to inspire and motivate
you.\textbackslash n\textbackslash nRemember, everyone has their
own creative style and pace. Just keep practicing and enjoying the
process of drawing.}%
\end{minipage}}

\noindent\shadowbox{\begin{minipage}[t]{1\columnwidth - 2\fboxsep - 2\fboxrule - \shadowsize}%
\texttt{\emph{Nectar}}

\texttt{\uline{Input:}}\texttt{ 0.002 = 1000 \textbackslash n1
= x?}

\texttt{\uline{Chosen Output:}}\texttt{ To find the value of x,
we can set up a proportion using the given information:\textbackslash n\textbackslash n0.002/1000
= 1/x\textbackslash n\textbackslash nTo solve for x, we can cross
multiply:\textbackslash n\textbackslash n0.002 {*} x = 1000 {*}
1\textbackslash n\textbackslash n0.002x = 1000\textbackslash n\textbackslash nDividing
both sides by 0.002:\textbackslash n\textbackslash nx = 1000 / 0.002\textbackslash n\textbackslash nx
= 500,000\textbackslash n\textbackslash nTherefore, 1 is equal to
500,000 in this proportion.}

\texttt{\uline{Rejected Output:}}\texttt{ 0.001 = x}%
\end{minipage}}

\begin{figure}
\begin{centering}
\includegraphics[width=0.8\textwidth]{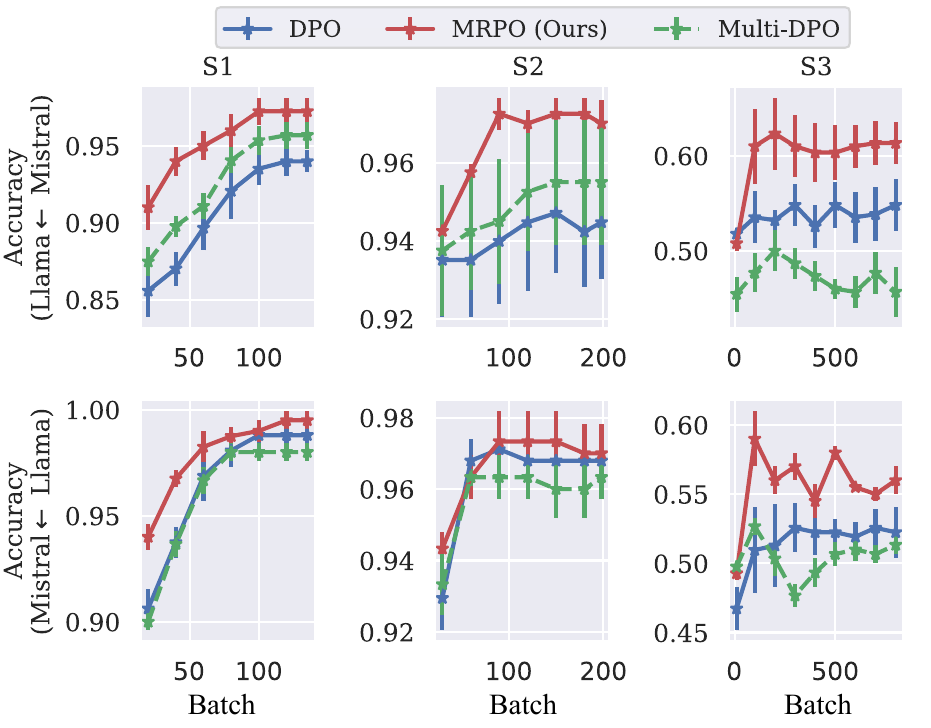}
\par\end{centering}
\caption{Chosen/Rejection preference accuracy on 3 small datasets: S1, S2 and
S3. The curves show mean and std. of preference accuracy on test sets
over training batches for 5 runs. In the first row, for MRPO and Multi-DPO,
RefM1 is LLama, and RefM2 is Mistral. In the second row, this order
is reversed. \label{fig:Chosen/Rejection-prediction-accu}}
\end{figure}

\begin{figure}
\begin{centering}
\includegraphics[width=0.95\textwidth]{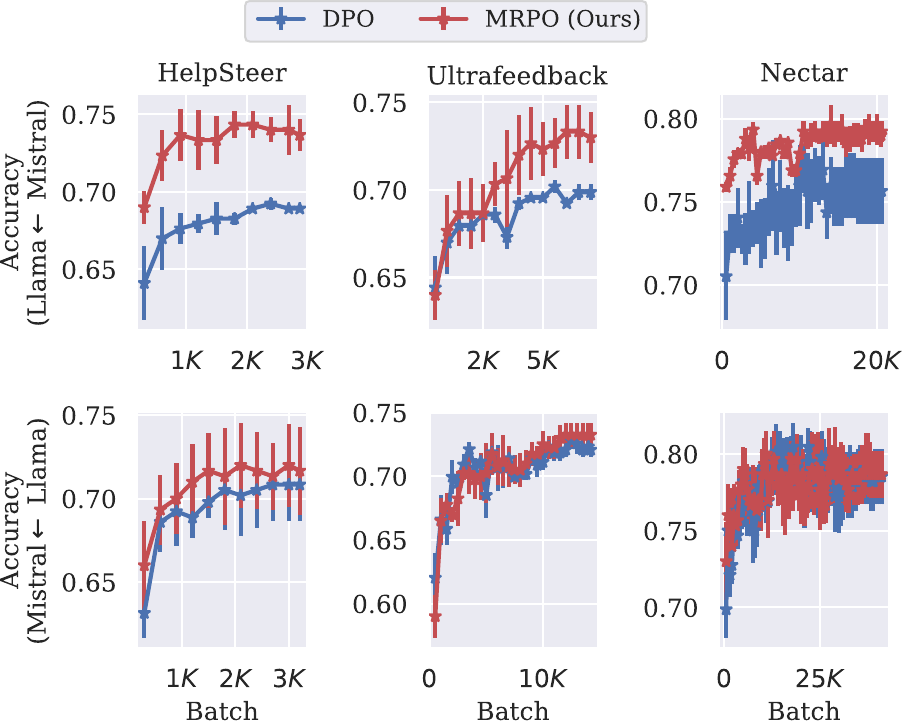}
\par\end{centering}
\caption{Chosen/Rejection preference accuracy on 3 big datasets: HelpSteer,
Ultrafeedback and Nectar. The curves show the mean and std. of
preference accuracy on test sets over training batches for 3 runs. \label{fig:3big}}
\end{figure}

\begin{figure}
\begin{centering}
\includegraphics[width=0.95\textwidth]{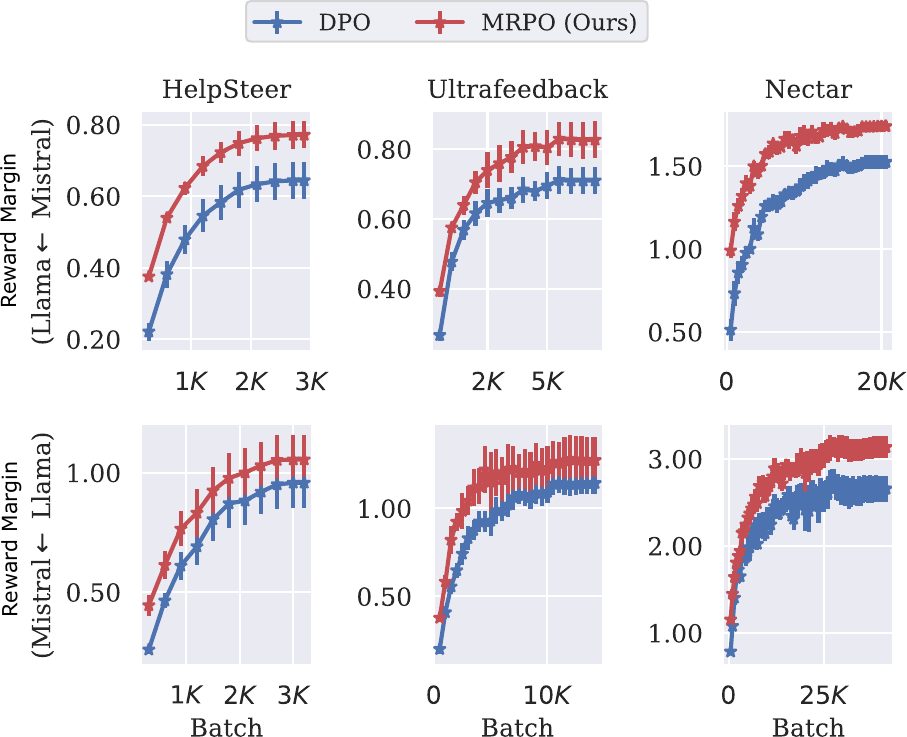}
\par\end{centering}
\caption{Reward Margin on 3 big datasets: HelpSteer,
Ultrafeedback and Nectar. The curves show the mean and std. of
reward margin on test sets over training batches for 3 runs. \label{fig:3bigmargin}}
\end{figure}

\begin{table}
\begin{centering}
\begin{tabular}{cccccccc}
\hline 
Dataset & GSM8K & TruthfulQA & HellaSwag & MMLU & Arc-easy & Winograd & Avg.\tabularnewline
\hline 
Base (Mistral) & 49.6 & 44.5 & 62.8 & 63.1 & 83.5 & 74.4 & 62.6\tabularnewline
DPO & 53.7 & \textbf{53.3} & \textbf{66.6} & 60.11 & 82.7 &\textbf{73.6} & 65.0 \tabularnewline
MRPO ($K=2$, M$\leftarrow$L) & \textbf{56.1} & 51.05 & 65.13 & 61.42 & \textbf{84.4} & 74.2 & \textbf{65.4}\tabularnewline
\hline
Base (LLama) & \textbf{23.9} & 37.8 & 57.8 & 46.4 & 60.8 & 66.4 & 51.0\tabularnewline
DPO & 22.0 & 39.5 & \textbf{59.4} & \textbf{46.3}	 & 73.5 & \textbf{67.3} & 51.4 \tabularnewline
MRPO ($K=2$, L$\leftarrow$M) & 22.4 & \textbf{45.3} & 58.1 & 46.4 & \textbf{74.1} & 66.8 & \textbf{52.2}   \tabularnewline
\hline 
\end{tabular}
\par\end{centering}
~

\caption{Test performance ($\times100$) across HuggingFace
Open LLM Leaderboard datasets when $K=2$. Bold denotes best.\label{tab:fullall2}}
\end{table}

\begin{figure}
\begin{centering}
\includegraphics[width=0.95\textwidth]{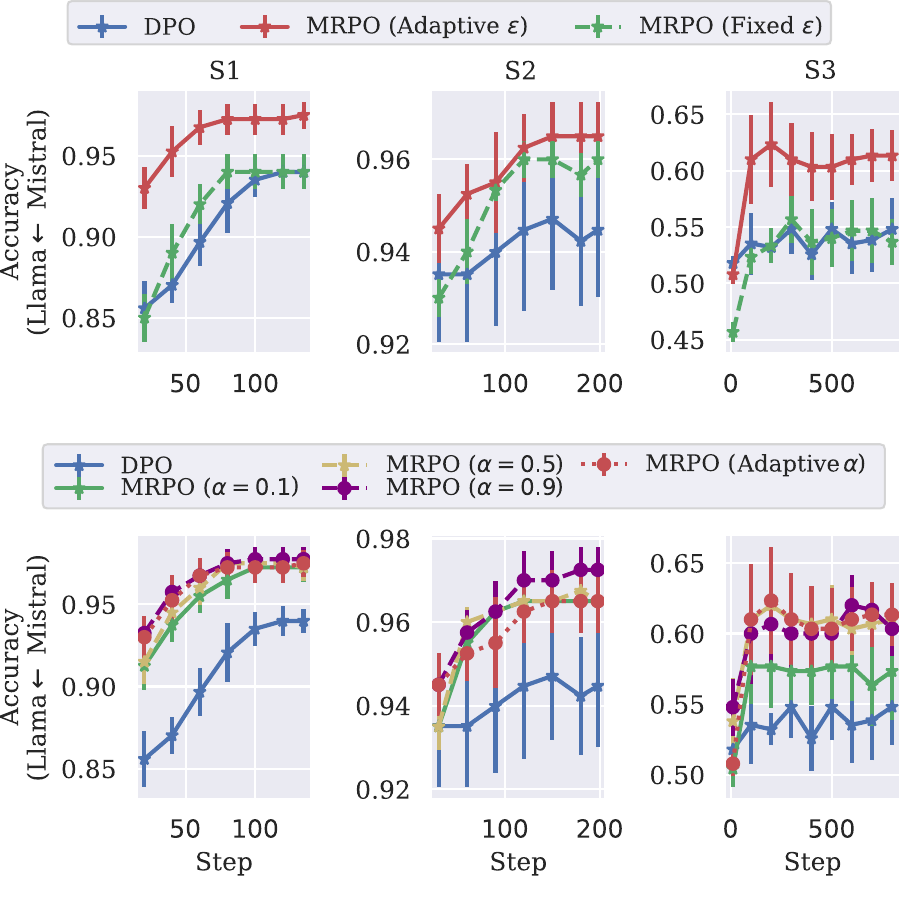}
\par\end{centering}
\caption{Analysis on $\epsilon$ and $\alpha$ using S1, S2 and S3 datasets.
The curves show the mean and std. of testing preference accuracy over
training batches for 5 runs. In the first row, adaptive $\epsilon_{max}=0.1$
is compared with fixed $\epsilon=0.1$. In the second row, adaptive $\alpha$
is compared with different fixed values of $\alpha$. \label{fig:Analysis-on-}}
\end{figure}

\end{document}